\documentclass{article}
\usepackage{ijcai25}

\usepackage{times}
\usepackage{soul}
\usepackage{url}
\usepackage[hidelinks]{hyperref}
\usepackage[utf8]{inputenc}
\usepackage[small]{caption}
\usepackage{graphicx}
\usepackage{amsmath}
\usepackage{amsthm}
\usepackage{booktabs}
\usepackage[switch]{lineno}

\newtheorem{example}{Example}
\newtheorem{theorem}{Theorem}

\usepackage{amssymb}
\usepackage{enumerate}
\usepackage[ruled,vlined,linesnumbered,nokwfunc,nofillcomment]{algorithm2e}
\SetKwInOut{In}{In}
\SetKwInOut{Out}{Out}
\SetKw{Break}{break}
\SetKw{Yield}{yield}
\SetKwRepeat{Do}{do}{while}

\newtheorem{proposition}{Proposition}
\newtheorem{corollary}{Corollary}
\newtheorem{lemma}{Lemma}
\theoremstyle{definition}
\newtheorem{definition}{Definition}
\newtheorem{remark}{Remark}

\newcommand{\cn}{\operatorname{Cn}}
\newcommand{\Bel}{\operatorname{Bel}}
\newcommand{\Gen}{\operatorname{Gen}}
\newcommand{\Com}{\operatorname{Com}}

\usepackage{color}

\title{On Definite Iterated Belief Revision with Belief Algebras\thanks{This is an extended version of the corresponding IJCAI 2025 paper with the same title.}}

\author{%
Hua Meng$^1$
\and
Zhiguo Long$^2$\thanks{Corresponding author.}\and
Michael Sioutis$^3$\and
Zhengchun Zhou$^1$ \\
\affiliations
$^1$School of Mathematics, Southwest Jiaotong University, China\\
$^2$School of Computing and Artificial Intelligence, Southwest Jiaotong University, China\\
$^3$LIRMM UMR 5506, University of Montpellier \& CNRS, France\\
\emails
\{menghua, zhiguolong\}@swjtu.edu.cn,
michael.sioutis@lirmm.fr,
zzc@swjtu.edu.cn
}

\begin{document}

\maketitle

\begin{abstract}
    Traditional logic-based belief revision research focuses on designing rules to constrain the behavior of revision operators. 
    Frameworks have been proposed to characterize iterated revision rules, but they are often too loose, leading to multiple revision operators that all satisfy the rules under the same belief condition. 
    In many practical applications, such as safety critical ones, it is important to specify a definite revision operator to enable agents to iteratively revise their beliefs in a deterministic way. 
    In this paper, we propose a novel framework for iterated belief revision by characterizing belief information through preference relations. Semantically, both beliefs and new evidence are represented as \emph{belief algebras}, which provide a rich and expressive foundation for belief revision. 
    Building on traditional revision rules, we introduce additional postulates for revision with belief algebra, including an upper-bound constraint on the outcomes of revision. We prove that the revision result is uniquely determined given the current belief state and new evidence. Furthermore,  to make the framework more useful in practice, we develop a particular algorithm for performing the proposed revision process. We argue that this approach may offer a more predictable and principled method for belief revision, making it suitable for real-world applications.
\end{abstract}

\section{Introduction}

Updating or revising the beliefs of an agent in light of new evidence is a fundamental process in both everyday life and scientific activities. For instance, Newton's laws of motion were widely accepted for centuries until discoveries at very small scales or very high speeds revealed their limitations. Similarly, our knowledge is continuously updated and enriched through learning and communication. To formalize this process, researchers in artificial intelligence have developed the subfield of \emph{belief change} (see, e.g., \cite{Doyle79,1976Rational}). Among the most influential contributions to this field is the AGM framework~\cite{alchourron1985logic}, which has inspired numerous extensions and applications in areas such as game theory and argumentation~\cite{Williams1996ApplicationsOB,vanharmelenHandbook2008,Hansson2011Special,DillerHLRW15,Zhang10}.

\paragraph{Motivation}

The AGM framework addresses various forms of belief change, including \emph{belief revision}, which is of particular interest in this paper. The AGM framework, along with its various subsequent developments, aims to update a current belief state to a new one when new evidence is acquired. This process can be formally described syntactically as revising a set of logical formulas (belief set) with a single logical formula (new evidence)~\cite{alchourron1985logic}. Alternatively, it can be characterized semantically using preferences over possible worlds (total preorder)~\cite{katsuno1991propositional}, where the revision process involves updating these preferences based on new evidence. Furthermore, belief revision based on partial preorders and iterated belief revision have been extensively studied~\cite{Lehmann95,boutilier1996iterated,darwiche1997logic,nayak2003dynamic,BoothC16,AravanisPW19,IsbernerHH24}. Researchers have primarily focused on how to constrain belief revision behaviors through postulates or how to characterize revision with new semantics, leading to the proposal of various rule systems~\cite{Liberatore24,Bonanno25}. These rule systems are evaluated from different perspectives to build a comprehensive framework. However, these rule systems are often too loose, resulting in multiple revision operators that all satisfy the same set of rules, which might not be desirable in practice. 
For example, in scenarios where multiple intelligent agents collaborate in urban traffic management, their initial beliefs are aligned. Upon acquiring new traffic condition information, these agents must update their traffic control strategies deterministically so that they maintain aligned beliefs, otherwise the transportation system might descend into chaos.

\paragraph{Contributions}

The goal and main contribution of this paper is not to establish a new rule system for belief revision, but to explore a simple and effective way to represent belief information in more depth, and to propose a definite revision operator suitable for applications requiring deterministic revision.
To this end, we use \emph{belief algebra} (introduced in~\cite{meng2015belief}) as the foundational tool for representing belief information. Unlike preorders over worlds, a belief algebra represents belief information as a preference relation over subsets of worlds. In our framework, both the current belief state and the new evidence are represented as belief algebras. Specifically, the iterated belief revision process is modeled as revising a belief algebra $G_1$ with another belief algebra $G_2$ to produce a new belief algebra $G_3$. By analyzing the structural properties of a belief algebra, we propose a set of revision postulates, including an upper bound constraint on the outcomes of revision, and prove that these postulates inherently induce a \emph{unique} revision operator. Additionally, we discuss the algorithmic implementation of this iterated belief revision framework, providing practical support for designing agents with belief revision capabilities in applications.

\paragraph{Organization}

The rest of the paper is organized as follows: Section 2 discusses related work and Section 3 introduces basic knowledge. Section 4 explores properties of belief algebras, and Section 5 considers a special case of revision with a belief algebra. Section 6 extends the discussion to general cases and Section 7 describes the practical algorithm. Finally, Section~8 concludes the paper.

  \section{Related Work}

  Various research efforts have focused on representing belief information in iterated revision. Spohn~\shortcite{spohn1988ordinal} introduced the concept of \emph{ordinal conditional function} (OCF) to encode preference information over worlds and developed a process called \emph{conditionalization} to revise OCFs. Williams~\shortcite{1994Transmutations} proposed a formula-based counterpart to OCF, mapping formulas to ordinals based on their resistance to change. Darwiche and Pearl~\shortcite{darwiche1997logic} advanced the field by representing belief information as total preorders on worlds and extending the AGM framework with four postulates to characterize iterated revision.
  
  Several researchers have improved the DP (Darwiche and Pearl) framework's settings. Benferhat et al.~\shortcite{benferhat2005revision} used partial preorders, while Peppas and Williams~\shortcite{peppas2014belief} employed semiorders. Ma et al.~\shortcite{ma2015a} revised epistemic states with partial epistemic states. Andrikopoulou et al.~\shortcite{andrikopoulou-2025} discussed belief revision under filters that are subsets of partially ordered sets. Benferhat et al.~\shortcite{benferhat2000iterated} enhanced the representation of new evidence by using an epistemic state, proposing postulates for minimal-model preserving operators and proving the uniqueness of the revision result given a total preorder and a new total preorder as evidence. This aligns with our discussion in Section~\ref{sec:rtot}, as their operator satisfies (RE1)--(RE3). Meanwhile, new frameworks and semantic structures continue to be proposed~\cite{Liberatore24,Bonanno25}.
  
  Many works aim to define the most ``reasonable'' revision rules, yet achieving consensus on rationality remains challenging. These debates often resemble philosophical discussions, focusing on abstract principles rather than practical implementations. 
  The rationality of the basic AGM rules has been questioned~\cite{Aravanis23}. 
  Sauerwald and Thimm~\shortcite{SauerwaldT24} considered the realizability of AGM revision and contraction operators in Epistemic Spaces.
  Some researchers, such as~\cite{booth2006admissible,jin2007iterated,nayak2003dynamic}, have observed that the DP framework can produce counter-intuitive revision results. To address that issue, they proposed modifications or additions to the DP postulates. For instance, Nayak, Pagnucco, and Peppas~\shortcite{nayak2003dynamic} introduced \emph{conjunction} postulates, treating consistent evidences as order-independent, while Jin and Thielscher~\shortcite{jin2007iterated} proposed a weaker \emph{independence} postulate. 
  Notably, these works primarily focus on rule construction, with limited attention to operator selection.
  
  While the exploration of revision rules continues, the potential and application prospects of belief revision have also received significant attention~\cite{HunterB22,BaroniFGS22}. This shift reflects a growing interest in leveraging belief revision for practical, real-world problems. In applications such as constrained differential privacy~\cite{LiuSZF23} and text generation by large language models~\cite{Bryan24}, the focus has shifted towards applications of revision. In industrial agent design, selecting a specific update algorithm is crucial to enable agents to iteratively refine their beliefs. To address this, our work focuses on representing belief information through \emph{belief algebra} building on~\cite{meng2015belief}, and proposes a unique revision operator by strengthening revision rules in a natural way, aiming to provide a deterministic and practical solution for real-world applications.

  \section{Preliminaries}\label{sec:pre}
We recall some necessary knowledge for what will follow.

  \subsection{Belief as A Total Preorder}
  
  In this paper, we restrict our discussion to belief revision in a finite propositional language $L$. We denote by $W$ the set of all (possible) worlds (i.e., interpretations).  
  
  For each propositional formula $\psi$, we denote by $[\psi]$ the set of all worlds of $\psi$, i.e., $[\psi] = \{\omega \in W \mid \omega \models \psi\}$. 
  We will also use a consequence operator $\cn(\Gamma) = \{\varphi\in L \mid [\Gamma] \subseteq [\varphi]\}$ to obtain the set of formulas implied by $\Gamma$.
  
  A (partial) \emph{preorder} $\preceq$ on $A$ is a binary relation on $A$ that is reflexive and transitive. A preorder $\preceq$ is called \emph{total} if any two elements in $A$ are comparable under $\preceq$.
  We  write $x\sim y$ if $x\preceq y$ and $y\preceq x$, and  $x\prec y$  if $x\preceq y$ but $y\not\preceq x$. 
  The \emph{strict part} of a preorder $\preceq$ is the set $\prec=\{(x,y) \mid x \prec y\} \subseteq \preceq$.
  
  A belief set $\mathcal{K}$ is a set of formulas in $L$ that is deductively closed, i.e., $\cn(\mathcal{K}) = \mathcal{K}$. Generally, belief revision is the process of changing a current belief state with a new piece of evidence, where the current belief state and the new evidence can be represented in different ways. For example, in the AGM framework, the current belief state is represented as a belief set and the new evidence is represented as a formula. 
  
  When belief revision needs to be done sequentially, known as \emph{iterated belief revision}, the representation mechanism of the AGM framework is not suitable anymore. 
  A more sophisticated structure known as \emph{epistemic state} is then used to represent belief information. 
  In their original paper, Darwiche and Pearl~\shortcite{darwiche1997logic} captured the concept of an epistemic state in terms of a revision operator. 
  Particularly, an epistemic state is a set of beliefs and \emph{conditional beliefs} satisfying several postulates.
  A conditional belief has the form $(\beta\mid\alpha)$, where $\alpha, \beta$ are formulas in $L$. An agent has a conditional belief $(\beta\mid\alpha)$ if she will believe $\beta$ whenever she believes $\alpha$. 
      Within the DP framework, an epistemic state can be characterized semantically as a total preorder on worlds as follows.
    \begin{lemma}[~\cite{darwiche1997logic}]Suppose $\Psi$ consists of a belief set $\Bel(\Psi)$ and several conditional beliefs, then $\Psi$ is an epistemic state iff there is a total preorder $\preceq$ on worlds such that:
      \begin{enumerate}
      \item[](ES1) $\phi\in \Bel(\Psi)$ iff there is a world $\omega$ in $[\phi]$ such that $\omega\prec\omega'$ for all $\omega'\in[\neg\phi]$.
      \item[](ES2) $(\beta\mid\alpha)\in \Psi$ iff there is a world $\omega$ in $[\alpha\wedge\beta]$ such that $\omega\prec\omega'$ for all $\omega'\in[\alpha\wedge\neg\beta]$.
      \end{enumerate}
      \end{lemma}

  \subsection{Belief as A Belief Algebra}
      To provide a unified representation of belief information, the concept of \emph{belief algebra} was introduced in~\cite{meng2015belief}. It is a class of ordering structures on $2^W$, the power set of possible worlds, which can intuitively capture the belief preference of an agent, and is actually more general than total preorder (see Section~\ref{sec:rtot}).
  
      \begin{definition}[\cite{meng2015belief}]\label{dfn:ba}
       Suppose $\gg$ is a binary relation on $2^{W}$, and write
       $R_W=\{(U,V)\mid U, V\subseteq W, U\cap V=\varnothing\}$.
       Then $(2^W,\gg)$ is called a \emph{belief algebra} (BA) if it satisfies the following rules ($U,V,U_1,V_1,U_2,V_2 \subseteq W$):
      \begin{itemize}
       \item[](A0)  $\gg\,\subseteq R_W$.
       \item[](A1)  $U\gg\varnothing$ iff $U\neq\varnothing$.
       \item[](A2)  If $U\gg V$, then $V \not\gg U$.
       \item[](A3)  If $U_1\supseteq U\gg V\supseteq V_1$ and $U_1\cap V_1=\varnothing$, then $U_1\gg V_1$.
       \item[](A4) If $U=U_1\cup V_1=U_2\cup V_2$ and $U_1\gg V_1$, $U_2\gg V_2$, then $U_1\cap U_2\gg V_1\cup V_2$.
      \end{itemize}
      \end{definition}
      Roughly speaking, $\gg$  directly describes the belief preference of agents in a semantic way. 
      For instance, $[\phi]\gg[\neg\phi]$ means that $\phi$ is more believable than $\neg\phi$, and if $[\phi]$ and $[\psi]$ are incomparable in $\gg$, then the agent 
      has no idea which one is more believable. (A0) shows that we only need to compare disjoint subsets of $W$. (A1) shows that each nonempty set has a higher preference level than the empty set. (A2) shows that $\gg$ is a strict ordering. (A3) shows that $\gg$ satisfies certain transitivity. (A4) considers the case where if $\phi$ is more believable than $\neg\phi$ and
      $\psi$ is more believable than $\neg\psi$, then $\phi\wedge\psi$ is more believable than $\neg(\phi\wedge\psi)$.
    
    Traditionally, the semantic characterization of an epistemic state is often represented as a \emph{total preorder} on $W$, where $W$ is the set of all possible worlds. In contrast, a \emph{belief algebra} is defined as an ordering relation on $2^W$, the power set of $W$. Interestingly, a total preorder on $W$ can be naturally extended to a belief algebra on $2^W$. This relationship can be formalized in the following theorem:
    
    \begin{theorem}[\cite{meng2015belief}]~\label{thm:cba}
      Let $\preceq$ be a total preorder on $W$. Define a binary relation $\gg$ on $2^W$ as follows: $U \gg V $\quad \text{if and only if} \quad $U \cap V = \emptyset$ \text{ and } $\exists \omega_1 \in U$ \text{ such that } $\forall \omega_2 \in V, \omega_1 \prec \omega_2.$ Then $(2^W, \gg)$ is a belief algebra.
    \end{theorem}
    
    This theorem demonstrates that the structure of a total preorder on $W$ can be lifted to a belief algebra on $2^W$, preserving the agent's belief preferences in a more expressive and generalized framework. The relation $\gg$ captures the intuition that a subset $U$ is preferred over $V$ if there exists at least one world in $U$ that is strictly preferred to all worlds in $V$. This extension provides a natural bridge between traditional epistemic states and the more general belief algebra framework.
    
    \begin{definition}[Complete Belief Algebra (CBA)]\label{dfn:cba}
    A belief algebra $(2^W, \gg)$ is called a \emph{complete belief algebra} (CBA) if it can be generated from a total preorder $\preceq$ on $W$ in the manner described in Theorem~\ref{thm:cba}.  
    \end{definition}
    Complete belief algebras provide a natural connection between total preorders on $W$, and the more expressive framework of belief algebras. They capture the intuition that belief preferences over subsets of worlds can be fully determined by a total preorder on individual worlds.
  \begin{example}
  Suppose $\omega_1\sim\omega_2\prec\omega_3\sim\omega_4$ is a total preorder and thus an epistemic state. Let $W=\{\omega_1,\omega_2,\omega_3,\omega_4\}$ and $Tr(W)=\{(U,\emptyset)\mid U\subseteq W, U\neq\emptyset\}$. For the sake of brevity we use $(x,y)$ for representing $\{\omega_x\}\gg\{\omega_y\}$, $(xy,z)$ for $\{\omega_x,\omega_y\}\gg\{\omega_z\}$, and so on. Then this total preorder is
  equivalent to a complete belief algebra $G$ as follows:
  \begin{align*}
  G &=\{(1,3),(1,4),(2,3),(2,4),(1,34), (2,34) \\
  &\quad (12,3), (12,4), (12,34),(123,4),(124,3)\} \cup Tr(W).
  \end{align*}
  \end{example}

  \section{Exploring Belief Algebra in More Depth}\label{sec:balg}

Given two belief algebras $G_1=(2^W,\gg_1)$ and $G_2=(2^W,\gg_2)$ on $W$, we denote by $G_1\subseteq G_2$ iff $\gg_1\subseteq\gg_2$. Also, we will not distinguish between $G_1\cup G_2$ and $\gg_1\cup \gg_2$, and $G_1\cap G_2$ and $\gg_1\cap \gg_2$, whenever it is self-explanatory based on context.
  For each belief algebra $G=(2^W,\gg)$, $Tr(W)=\{(U,\emptyset)\mid U\subseteq W, U\neq\emptyset\}$ is always contained in $\gg$ by $(A1)$. 
  It is not difficult to verify that $(2^W,Tr(W))$ is a belief algebra. Then for each belief algebra $G$ on $W$, we always have $(2^W,Tr(W))\subseteq G$. 
  
  A subset of $R_W$ (defined in Definition~\ref{dfn:ba}) can generate a belief algebra.
  \begin{definition}\label{dfn:gen}
   Given $\Omega\subseteq R_W$, we denote by $\Gen(\Omega)$ the smallest subset of $R_W$ that contains $\Omega$ and is closed under the rules (A1), (A3) and (A4) used for expansion. 
  \end{definition}
  \begin{example}\label{eg:gen}
  Suppose $W=\{\omega_1,\omega_2,\omega_3,\omega_4\}$ and $\Omega=(\{\omega_1,\omega_2\},\{\omega_3,\omega_4\})$. Then for each belief algebra $(2^W,\gg)$, if $\Omega\in\gg$ we have $\{\omega_1,\omega_2\}\gg\{\omega_3,\omega_4\}$. Again, for the sake of brevity we use $(x,y)$ for representing $\{\omega_x\}\gg\{\omega_y\}$, $(xy,z)$ for $\{\omega_x,\omega_y\}\gg\{\omega_z\}$, and so on. 
  With respect to~(A3) we have that $(123,4),(124,3)$ are all in $\gg$, and by (A4) $(12,34)\in \gg$. Then it can be verified that $\Gen(\Omega)=\{(12,34),(123,4),(124,3)\}\cup Tr(W)$ is the smallest belief algebra which contains $\Omega$. 
  \end{example}
  If $\Gen(\Omega)$ also satisfies (A2), then $(2^W,\Gen(\Omega))$ is a belief algebra by definition, and we will also use $\Gen(\Omega)$ to denote this belief algebra. Here we always assume that the agent is rational and $\Gen(\Omega)$ is always a belief algebra when $\Omega$ represents some belief information. Given $\Omega$, $\Gen(\Omega)$ can be obtained by closing $\Omega$ under (A1), (A3), and (A4).

  With the operator $\Gen(\cdot)$, one can see that each CBA is entirely determined by preferences on the sets consisting of a single world.

  There is a 1-1 correspondence between CBAs and total preorders, which is given as follows.
  \setcounter{lemma}{3}
  \begin{lemma}[\cite{meng2015belief}]\label{pro:cba-to}
    There is a 1-1 correspondence between CBAs and total preorders:
    \begin{itemize}
      \item Suppose $\preceq$ is a total preorder on $W$. 
      Let $U \gg V$ iff $\exists \omega_1 \in U$ s.t. $\forall \omega_2 \in V$, $\omega_1 \prec \omega_2$. Then $(2^W, \gg)$ is a CBA.
      \item Suppose $(2^W, \gg)$ is a CBA. Let $\omega_1 \prec \omega_2$ iff $\{\omega_1\} \gg \{\omega_2\}$, and $\omega_1 \sim \omega_2$ iff $\{\omega_1\} \not\gg \{\omega_2\}$ and  $\{\omega_2\} \not\gg \{\omega_1\}$. Then $\preceq := \prec \cup \sim$ is a total preorder on $W$.
    \end{itemize}
  \end{lemma}
  Each CBA is totally decided by preferences on the sets consisting of a single world, which is implied straightforwardly by Lemma~\ref{pro:cba-to}.
  \begin{corollary}\label{cor:cba-to}
    Suppose $G=(2^W,\gg)$ and $G'=(2^W,\gg')$ are CBAs, then:
    \begin{itemize}
      \item $\Gen(\Omega)=G$, where $\Omega=\{(\{\omega\},\{\omega'\})\mid \{\omega\}\gg\{\omega'\}\}$.
      \item If $G$ and $G'$ have the same preferences on single world sets,
      i.e., for any $\omega,\omega'\in W$, $(\{\omega\},\{\omega'\})\in G$ iff $(\{\omega\},\{\omega'\})\in G'$, 
      then $G=G'$.
    \end{itemize}
  \end{corollary}

  The following proposition shows how to construct new belief algebras from existing belief algebras.
  
  \begin{proposition}\label{thm:cap}
  Suppose $G=(2^W,\gg)$ and  $G'=(2^W,\gg')$ are belief algebras. Then
  $G\cap G'=(2^W,\gg\cap\gg')$ is a belief algebra.
  \end{proposition}
  \begin{proof}
    We only need to show that $\gg\cap\gg'$ satisfies (A0)-(A4). We take (A3) as an example, as the rest can be proven in similar and/or simpler fashion. If $(U,V)\in \gg\cap\gg'$, $U_1\supseteq U , V\supseteq V_1$ and $U_1\cap V_1=\varnothing$, then $U_1\gg V_1$ and $U_1\gg V_1$ since $G_1,G_2$ satisfy (A3).
    Hence $(U,V)\in \gg\cap\gg'$. This means that $G\cap G'$ satisfies (A3).
    \end{proof}
  \begin{corollary}\label{cor-subalgebra}
  Suppose $G=(2^W,\gg)$ is a belief algebra and $\Omega\subseteq\gg$,  and $\mathcal{A}$ is the set of all the belief algebras that contains $\Omega$. Then $\Gen(\Omega)$ is a belief algebra, and $\Gen(\Omega)=\bigcap\mathcal{A}$.
  \end{corollary}
  \begin{proof}
  By Definition~\ref{dfn:gen} and $\Omega\subseteq\  \gg$,  we have $\Gen(\Omega)\subseteq G$. Then $\Gen(\Omega)$ satisfies (A2), and $\Gen(\Omega)$ is the smallest belief algebra containing $\Omega$.
  Thus, $\Gen(\Omega) \in \mathcal{A}$ and $\bigcap \mathcal{A} \subseteq \Gen(\Omega)$. 
  On the other hand,
  by Proposition~\ref{thm:cap}, we know that any finite intersection of belief algebras is also a belief algebra, and $\mathcal{A}$ contains only finite number of belief algebras because $W$ is a finite set, so $\bigcap\mathcal{A}$ is a belief algebra containing $\Omega$. Then $\Gen(\Omega) \subseteq \bigcap\mathcal{A}$ and we have $\Gen(\Omega)=\bigcap\mathcal{A}$.
  \end{proof}
  
  The following result unveils the structure of $\gg$, which consists of levels of subsets of worlds.  
  \begin{lemma}[\cite{meng2015belief}]\label{prop:chain}
  Suppose $(2^W,\gg)$ is a belief algebra. Then there is a unique chain, called \emph{backbone}, $\Delta=\{U_1\gg U_2\gg U_3\gg\cdots\gg U_n\}$, such that:
  \begin{itemize}
  \item[]\rm (Ch1)\quad $\{U_i\}_{i=1}^n$ is a partition of $W$, i.e., $\{U_i\}_{i=1}^n$ consists of pairwise disjoint nonempty subsets of $W$ and $\bigcup_{i=1}^{n} U_i=W$.
  \item[]\rm (Ch2)\quad For each $U_i$, if $V_1, V_2$ are two disjoint nonempty subsets of $U_i$, then $V_1$ and $V_2$ are incomparable in $\gg$, i.e., $(V_1, V_2)\notin \gg$ and $(V_2, V_1)\notin \gg$.
  \end{itemize}
  \end{lemma}
   The above lemma shows that the backbone is the core structure of a belief algebra. (Ch2) shows that for each $U_i$ in a backbone, any subsets of $U_i$ can not be compared with one another. In terms of belief preferences, this means that an agent has no preference on subsets of $U_i$. 
   Another important concept is the \emph{support} of a subset of $W$ w.r.t. some backbone.
   \begin{definition}
    Let $\Delta=\{U_1\gg U_2\gg U_3\gg\cdots\gg
    U_n\}$ be the backbone of a belief algebra. 
    The \emph{support} of  a nonempty set $V\subseteq W$ w.r.t. $\Delta$ is defined as
    $I(V)=U_i$, where $V \cap U_i \ne \emptyset$ and $\forall j \ne i$, if $V \cap U_j \ne \emptyset, U_i \gg U_j$. That is, $I(V)$ is the largest $U_i$ under $\gg$ in the backbone such that $V \cap U_i \ne \emptyset$.
  \end{definition}
  \begin{figure}[tb]
    \centering
    \includegraphics[width=0.6\linewidth]{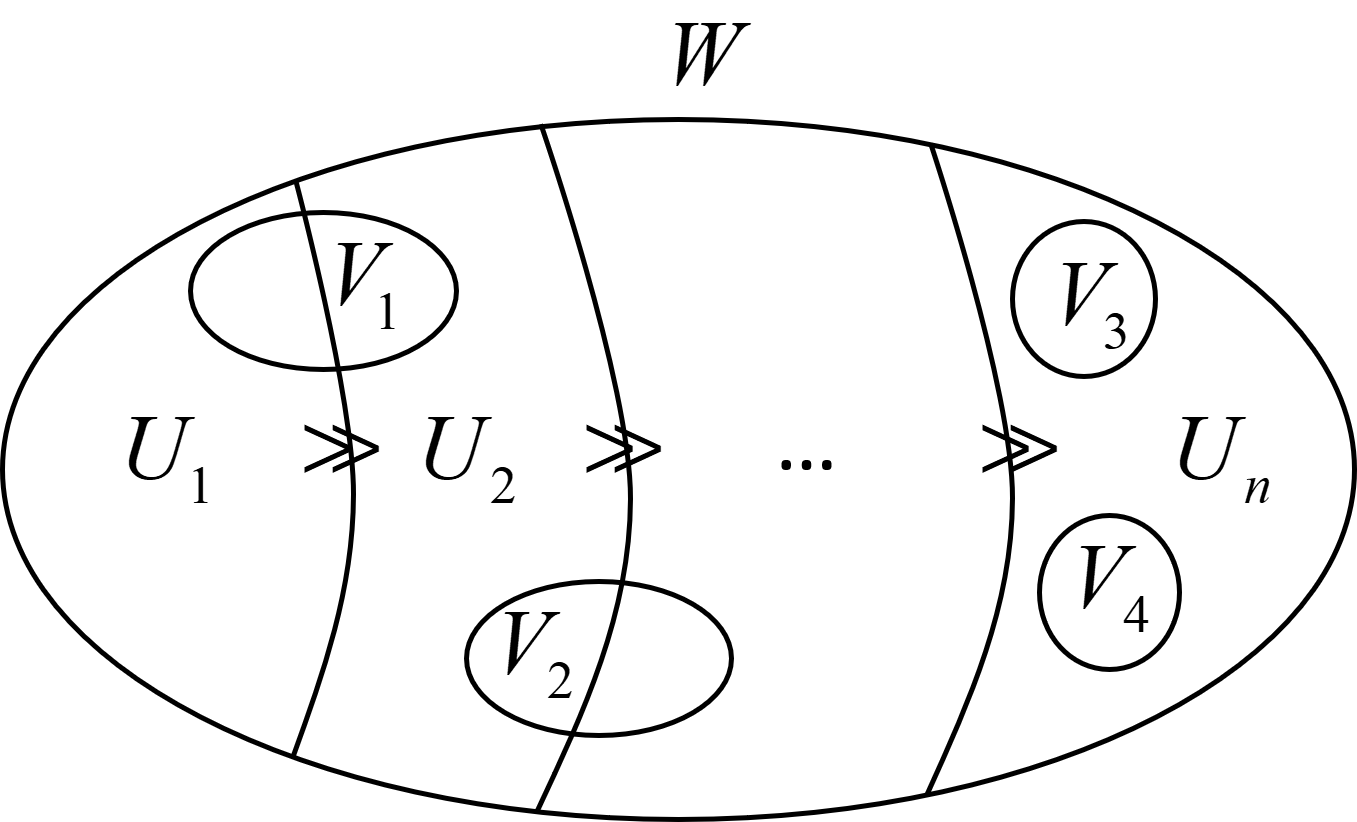}
    \caption{Illustration of the structure of the backbone $\{U_1 \gg U_2 \gg \cdots \gg U_n\}$ and supports of a belief algebra. The backbone forms a partition of $W$, and according to (Ch2), $V_3 \not\gg V_4$ and $V_4 \not\gg V_3$, because $V_3 \cap V_4 = \emptyset$ and $V_3,V_4 \subseteq U_n$. Here,
    the supports w.r.t. the backbone are $I(V_1) = U_1$, $I(V_2) = U_2$, and
    $I(V_3)=I(V_4) = U_n$.}\label{fig:bs}
  \end{figure}
  Figure~\ref{fig:bs} gives an illustration of the above definitions. 
   Moreover, the preferences of an agent are ``consistent'' with the backbone in the sense of property (1) in Lemma~\ref{lem:sup} that follows.
  \begin{lemma}[\cite{meng2015belief}]\label{lem:sup}
  The relations between belief algebras and their backbones are as follows:
  \begin{enumerate}
  \item[(1)] Suppose $G=(2^W,\gg)$ is a belief algebra, then $U\gg V$ only if $I(U)\gg I(V)$.
  \item[(2)] Suppose $G=(2^W,\gg)$ is a complete belief algebra, then $U\gg V$ if and only if $I(U)\gg I(V)$.
  \item[(3)] Suppose $G=(2^W,\gg)$ is a belief algebra, and $\Delta=\{U_1\gg U_2\gg U_3\gg\cdots\gg U_n\}$ is its backbone. Then there is a unique complete belief algebra \mbox{$G'=(2^W,\gg')$} containing $G$ and having the same backbone with $G$. Moreover, $\gg'=\{(U,V)\in R_W\mid U,V\neq\emptyset, I(U)\gg I(V)\}\cup Tr(W)$.
  \end{enumerate}
  \end{lemma}
  The above lemma shows that there are different belief algebras that have the same backbone $\Delta$, and there is a unique complete belief algebra with $\Delta$ as its backbone, viz., the largest one containing all the pairs $(U,V)$ s.t. $I(U) \gg I(V)$. Roughly speaking, the backbone of a belief algebra is the ``core'' belief information that reflects the main preferences of the agent. 
  Taking Figure~\ref{fig:bs} as an example, a belief algebra $G$ with $\{U_1\gg U_2\gg U_3\gg\cdots\gg U_n\}$ as its backbone may or may not contain $V_1 \gg V_2$, but if it is CBA, then it must contain $V_1 \gg V_2$, and it should never contain $V_2 \gg V_1$ because $I(V_1)=U_1 \gg U_2= I(V_2)$.
  \begin{definition}\label{def:com}
  Suppose $G=(2^W,\gg)$ is a belief algebra. Then we denote by $\Com(G)$ the complete belief algebra that contains $G$ and has the same backbone as $G$. Suppose $G_1$ and $G_2$ are both belief algebras. Then we write $G_1\leq G_2$ if these two belief algebras have the same backbone and $G_1\subseteq G_2$.
  \end{definition}
\begin{example}\label{eg:dp}
  Let $L$ be a propositional language with two variables $\{b,f\}$ and $W=\{\omega_1  := b\wedge f, \omega_2 := b\wedge \neg f, \omega_3 := \neg b\wedge f, \omega_4 := \neg b\wedge \neg f\}$. Suppose Bob's current belief state is represented as a total preorder $\omega_1\sim\omega_2\prec\omega_3\sim\omega_4$, and the new evidence is a formula $\mu$ such that the worlds that entail $\mu$ are represented as
  $[\mu]=\{\omega_1,\omega_4\}$. 
  In this situation, Bob holds a new preference that $[\mu] \gg [\neg\mu]$.
These belief preferences can be represented using belief algebra. 
Following the notation in Example~\ref{eg:gen}, Bob's current belief information can be represented as: $Gen(\{(1, 3), (1, 4), (2, 3), (2, 4)\}).$ Similarly, the new evidence $\mu$ can be represented as $Gen(\{(14, 23)\})$. Note that $Gen(\{(1, 3), (1, 4), (2, 3), (2, 4)\})$ is a \emph{complete belief algebra} (CBA), while $\Gen(\{(14, 23)\})$ is not a CBA. The backbone of the former is $\{\omega_1, \omega_2\} \gg \{\omega_3, \omega_4\}$, while the backbone of the latter is $\{\omega_1, \omega_4\} \gg \{\omega_2, \omega_3\}$. Moreover, $Com(Gen(\{(14, 23)\}))=Gen(\{(1, 2), (1, 3), (4, 2), (4, 3)\})$.
\end{example}
  Note that $G_1 \leq G_2$ is different from $G_1\subseteq G_2$. It actually means that $G_1$ and $G_2$ contain the same ``core'' belief information, but $G_1$ is ``less informational'' than $G_2$, in the sense that any revision result of $G_1$ should be contained in that of $G_2$, as we will see later in the postulate (RA5). Let $BA_L$ be the set of belief algebras over the worlds of $L$. It is easy to check that $\leq$ as defined above is a partial order on $BA_L$. We also use $BA_L$ to denote the partial order set $(BA_L,\leq)$.
  
  \begin{theorem}\label{thm:lat}
    Let $\Delta=\{U_1\gg U_2\gg U_3\gg\cdots\gg U_n\}$ be a backbone and denote by $BA(\Delta)=\{G\in BA_L\mid \Delta$ is the backbone of $G\}$ the set of belief algebras having $\Delta$ as backbone.
  Suppose $G_1, G_2\in BA(\Delta)$. Then:
  \begin{enumerate}[{(1)}]
  \item $G_1\cap G_2\in BA(\Delta) $.
  \item $\Gen(G_1\cup G_2)$ is a belief algebra, and $\Gen(G_1\cup G_2)\in BA(\Delta)$.
  \end{enumerate}
  \end{theorem}
  \begin{remark}
  Theorem~\ref{thm:lat} shows that $BA(\Delta)$ is a lattice that contains the smallest element $\Gen(\Delta)$ and the largest element $\Com(\Gen(\Delta))$. Here, a lattice is a set equipped with a partial order such that every two elements have a unique supremum (also called the least upper bound) and a unique infimum (also called the greatest lower bound).  Furthermore, $BA_L$ is divided into disjoint parts by the backbone, and each part is a lattice under $\leq$ (see Figure~\ref{fig:bal}). 
  \end{remark}

  \begin{figure}[tb]
    \centering
    \includegraphics[width=0.7\linewidth]{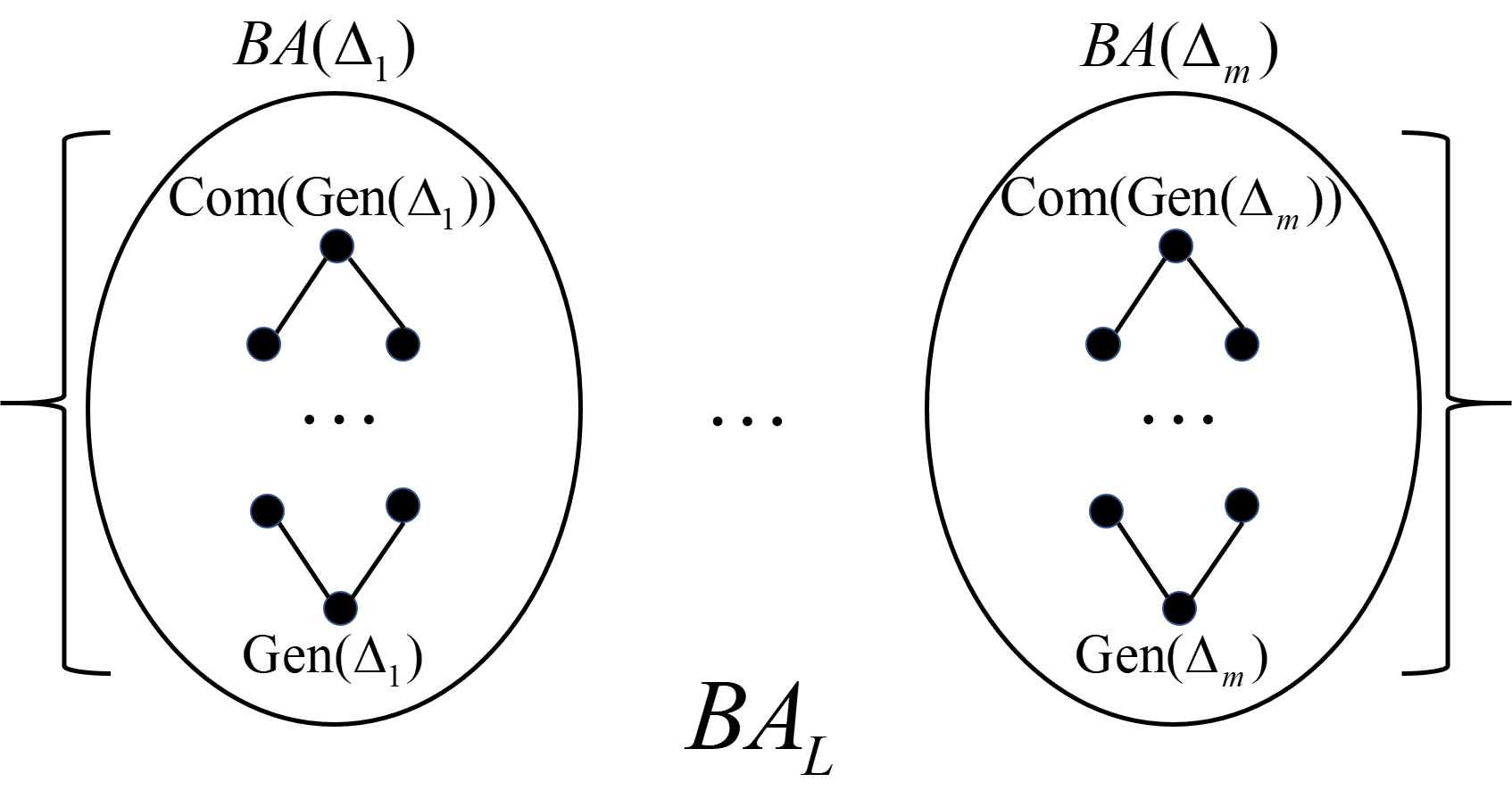}
    \caption{Illustration of the structure of $BA_L$.}\label{fig:bal}
  \end{figure}
  \begin{example}
  Suppose $W=\{\omega_1, \omega_2, \omega_3, \omega_4\}$. Let $G_1=\{(123,4)\}\cup Tr(W)$,
  $G_2=\{(123, 4), (12, 4)\}\cup Tr(W)$, $G_3=\{(123, 4), (13,4)\}\cup Tr(W)$, $G_4=\{(123, 4), (12,4), (13,4), (23,4), (1, 4), (2, 4), (3, 4)\}\cup Tr(W)$. Then $G_1, G_2, G_3, G_4$ are all belief algebras, and they all have the same backbone $\Delta=\{\omega_1, \omega_2, \omega_3\}\gg\{\omega_4\}$. $G_4$ is a complete belief algebra, and $G_1\leq G_2\leq G_4$, $G_1\leq G_3\leq G_4$, but $G_2, G_3$ can not be compared by $\leq$. Furthermore, $G_1=\Gen(\Delta)$ is the smallest element in $BA(\Delta)$, and $G_4=\Com(G_1)$ is the biggest one.
  \end{example}

  \section{Special Case: Revising CBA with CBA}\label{sec:rtot}
Before discussing the general case of revision with belief algebras, we first consider a simplified scenario where the agent's current belief, the new evidence, and the revision result are all represented as \emph{complete belief algebras} (CBAs; see Definition~\ref{dfn:cba}). This process is equivalent to revising an existing total preorder $\preceq_1$ by a new total preorder $\preceq_2$ to obtain a revised total preorder $\preceq_3$.

Inspired by existing research (e.g., \cite{alchourron1985logic,darwiche1997logic,ma2015a}), we propose the following revision rules for this setting. The revision of a total preorder by another total preorder, denoted as $\preceq_1 \circ \preceq_2 = \preceq_3$, has been studied in~\cite{benferhat2000iterated} and is known as a \emph{minimal-model preserving operator}. This operator can be equivalently characterized by the following postulates:
\begin{enumerate}
  \item[](RE1) $\preceq_3$ is a total preorder.
  \item[](RE2) $\prec_2 \subseteq \prec_3$, where $\prec_i=\{(x,y) \mid x \preceq_i y \wedge y\not\preceq_i x\}$.
  \item[](RE3) If $\omega \sim_2 \omega'$, then: $\omega \prec_3 \omega'$ if and only if $\omega \prec_1 \omega'$.
\end{enumerate}
The postulate (RE1) is natural by underlying setting. The postulate (RE2) requires that the new belief information (corresponding to the strict part of the new total preorder) must be fully incorporated into the revision result. Finally, (RE3) ensures that for worlds with equal (non)preference under $\preceq_2$, the preference relation from $\preceq_1$ is preserved in $\preceq_3$. This allows the revision process to retain more of the original belief information while incorporating the new evidence.
    \begin{example}~\label{eg:t-t}
      Following Example~\ref{eg:dp}, Bob's current belief is  $\preceq_1=\{\omega_1\sim\omega_2\prec\omega_3\sim \omega_4\}$ which has $\{(\omega_1 \prec \omega_3),(\omega_1 \prec \omega_4),(\omega_2 \prec \omega_3),(\omega_2 \prec \omega_4)\}$ as its strict part.
    Now, instead of considering revision by a formula, we consider revision by another total preorder $\preceq_2=\{\omega_1\sim\omega_2\sim\omega_3\sim\omega_4\}$. 
    Then $\preceq_2$ 
    has an empty strict part, which means that $\preceq_2$ cannot lead to any new preference. 
    In this sense, all the strict orderings of $\preceq_1$ can be kept into the revision result. It is natural that $\preceq_1\circ\preceq_2=\preceq_1$. Suppose instead $\preceq_2=\{\omega_4\prec\omega_3\prec\omega_2\prec\omega_1\}$. Then, in $\preceq_2$, all worlds are ordered by $\prec$. In this sense, no information in $\preceq_1$ is useful and $\preceq_1\circ\preceq_2=\preceq_2$. On the other hand, if  $\preceq_2=\{\omega_2\sim\omega_4\prec\omega_1\sim\omega_3\}$. Then, $\omega_2$ and $\omega_4$ ($\omega_1$ and $\omega_3$
    respectively) can not be distinguished in $\preceq_2$, but $\omega_2\prec\omega_4$ and $\omega_1\prec\omega_3$ are in $\preceq_1$. Following the information in $\preceq_1$, Bob should hold $\omega_2\prec\omega_4$ and $\omega_1\prec\omega_3$ in $\preceq_1\circ\preceq_2$. Then $\preceq_1\circ\preceq_2=\{\omega_2\prec\omega_4\prec\omega_1\prec\omega_3\}$.
   \end{example}

  Interestingly, it can be shown that revision results satisfying the above postulates are unique.
  \begin{theorem}\label{thm:totuni}
  Suppose $\preceq_1$ and $\preceq_2$ are total preorders. Then there is a unique revision operator satisfying (RE1)-(RE3).
  \end{theorem}
  \begin{proof}
    For any $\omega,\omega'\in W$, we define a binary relation $\preceq$ on $W$ as follows.
    \begin{itemize}
    \item $\omega\prec\omega'$ iff $\omega\prec_2\omega'$, or $\omega\sim_2\omega'$ and $\omega\prec_1\omega'$.
    \item $\omega\sim\omega'$ iff $\omega\sim_1\omega'$ and $\omega\sim_2\omega'$.
    \end{itemize}
    It is not difficult to check that $\preceq$ is a total preorder. Next we only need to show that if an operator $\circ$ satisfies (RE1)-(RE3) then the revision result of $\preceq_1\circ\preceq_2$ is exactly $\preceq$. Let $\preceq_3=\preceq_1\circ\preceq_2$. By (RE1), $\preceq_3$ is a total preorder. Since $\preceq_2$ is a total preorder, for any $\omega,\omega'\in W$, exactly one of $\omega\prec_2\omega'$, $\omega'\prec_2\omega$, and $\omega \sim_2\omega$ will happen. If $\omega\prec_2\omega'$ ($\omega'\prec_2\omega$, respectively) then $\omega\prec_3\omega'$ ($\omega'\prec_3\omega$, respectively) by (RE2). If $\omega\sim_2\omega'$ then we have
     $\omega\prec_3\omega'$ iff $\omega\prec_1\omega'$ by (RE3). That is equivalent to say, if $\omega\sim_2\omega'$ and $\omega\prec_1\omega'$, then $\omega\prec_3\omega'$, and if $\omega\sim_2\omega'$ and $\omega\sim_1\omega'$, then $\omega\sim_3\omega'$. In the end, we have $\omega\prec_3\omega'$ iff $\omega\prec\omega'$, and $\omega\sim_3\omega'$ iff $\omega\sim\omega'$. Therefore, $\preceq_3=\preceq$.
    \end{proof}

  \begin{remark}
  However, 
  the above revision framework has the great limitation of not being able to deal with more general cases such as the new evidence or even the current belief state is not a total preorder, which is common in real-world applications, e.g., the agent only holds incomplete belief information. Therefore, it is necessary to consider other solutions without such limitation.
  \end{remark}
   
  \section{Revision with Belief Algebras}\label{sec:rba}
  In this section, we extend our discussion to a more general framework of iterated belief revision based on a \emph{belief algebra}, rather than restricting ourselves to CBAs. This framework allows for a richer representation of belief states and supports more flexible revision processes. For simplicity, we assume that the agent is rational (i.e. both her belief and the new evidence do not contain ``conflicting information''). If not specified otherwise, we always suppose that the agent's current belief state is a belief algebra $G_1$, the new evidence is a belief algebra $G_2$, and the revision result is also a belief algebra denoted by $G_1\bullet G_2$, where $\bullet$ is a revision operator from $BA_L\times BA_L$ to $BA_L$. 

  To give postulates to characterize rational revisions, following literature, e.g., \cite{alchourron1985logic,darwiche1997logic,jin2007iterated}, we first assume that new evidence has a higher preference, that is, belief information of $G_2$ is more believable than $G_1$. Then $G_2$ should be kept in $G_1\bullet G_2$, and we have:
  \begin{enumerate}
  \item[](RA1) $G_2\subseteq G_1\bullet G_2$
  \end{enumerate}

  Since $G_1$ and $G_2$ collectively cover all the agent's belief information, we assume that $G_1 \bullet G_2$ is entirely determined by $G_1$ and $G_2$ and is generated by some subset of $G_1 \cup G_2$. This leads to the following postulate:  
  \begin{enumerate}
  \item[](RA2) There is an $\Omega\subseteq G_1\cup G_2$ s.t. $G_1\bullet G_2=\Gen(\Omega)$.
  \end{enumerate}
  (RA2) requires that the revision result cannot be generated with information outside $G_1 \cup G_2$, and if $G_1$ and $G_2$ have conflicting information, one should choose a consistent subset in order to generate a belief algebra.

  Notice that a complete belief algebra is equivalent to a total preorder on worlds. Revising a complete belief algebra $G_1$ by another complete belief algebra $G_2$ is equivalent to revising a total preorder by another total preorder. 
  The following postulates (RA3) and (RA4) are thus inspired by (RE1) and (RE3) for revising total preorders in Section~\ref{sec:rtot}, respectively.
  \begin{enumerate}
  \item[](RA3) Suppose $G_1$ and $G_2$ are complete belief algebras. Then $ G_1\bullet G_2$ is also a complete belief algebra.
  \item[](RA4) Suppose $G_1=(2^W,\gg_1)$ and $G_2=(2^W,\gg_2)$ are complete belief algebras,  and $I_2(\{\omega\})=I_2(\{\omega'\})$ in $G_2$. Then $(\{\omega\},\{\omega'\})\in G_1\bullet G_2$ iff $\{\omega\} \gg_1 \{\omega'\}$. 
  \end{enumerate}

  \begin{theorem}\label{thm:rcba}
    Suppose $\bullet$ satisfies (RA1)--(RA4), and $G_1=(2^W,\gg_1)$ and $G_2=(2^W,\gg_2)$ are complete belief algebras. $\Lambda(G_1,G_2)=\{(\{\omega\},\{\omega'\})\mid \{\omega\}\gg_1\{\omega'\},I_2(\{\omega\})=I_2(\{\omega'\})\}$.
    Then $G_1\bullet G_2=\Gen(\Lambda(G_1,G_2)\cup G_2)$, and the result of revising $G_1$ by $G_2$ is unique.
  \end{theorem}
  \begin{proof}
    Suppose the corresponding total preorders of $\gg_1,\gg_2$ are $\preceq_1,\preceq_2$ respectively. Let
   $\preceq=\preceq_1\circ\preceq_2$, where $\circ$ is the operator which satisfies (RE1)-(RE3). Then $\preceq$ is a total preorder, and $\omega\prec\omega'$ iff $\omega\prec_2\omega'$, or $\omega\sim_2\omega'$ and $\omega\prec_1\omega'$. Since $G_1,G_2$ are CBAs, we have $\omega\prec_2\omega'$ iff $\{\omega\}\gg_2\{\omega'\}$, $\omega\prec_1\omega'$ iff $\{\omega\}\gg_1\{\omega'\}$, and  $\omega\sim_2\omega'$ iff $I_2(\{\omega\}=I_2(\{\omega'\})$. We denote by $G=(2^W,\gg)$ the corresponding CBA of $\preceq$. Then $\{\omega\}\gg\{\omega'\}$ iff $\omega\prec\omega'$. That is to say,
   $\{\omega\}\gg\{\omega'\}$ iff $\{\omega\}\gg_2\{\omega'\}$, or $I_2(\{\omega\})=I_2(\{\omega'\})$ and $\{\omega\}\gg_1\{\omega'\}$. In summary, for any $\omega,\omega'\in W$, $(\{\omega\},\{\omega'\})\in G$ iff 
   $(\{\omega\},\{\omega'\})\in G_2\cup\Lambda(G_1,G_2)$. The ``if'' part here shows $\Lambda(G_1,G_2)\subseteq G$ and 
    $G_2\subseteq G$, because $G_2$ is a CBA and each CBA is totally decided by the preferences on single world sets according to Corollary~\ref{cor:cba-to}. Then $\Gen(G_2\cup\Lambda(G_1,G_2))\subseteq G$. The ``only if '' part here shows that if $(\{\omega\},\{\omega'\})\in G$, then we have $(\{\omega\},\{\omega'\})\in G_2\cup\Lambda(G_1,G_2)$, which means $G\subseteq \Gen(G_2\cup\Lambda(G_1,G_2))$ since $G$ is also a CBA. Therefore, we have $G= \Gen(G_2\cup\Lambda(G_1,G_2))$.
    
    To show that the result is unique,
  suppose that $\bullet$ is an operator on $BL_A$ which satisfies (RA1)--(RA4). Then we only need to show that if $G_1\bullet G_2=G_3$ then $G_3=G$. 
  By (RA3), $G_3$ is a CBA because $G_1$ and $G_2$ are CBAs. 
  Let $G_3=(2^W,\gg_3)$. Then $\forall \omega,\omega'\in W$, one of the following two cases is true.
  \begin{itemize}
  \item[](Case 1) If $\{\omega\}\gg_2\{\omega'\}$ or $\{\omega'\}\gg_2\{\omega\}$ then we have $\{\omega\}\gg_3\{\omega'\}$ or $\{\omega'\}\gg_3\{\omega\}$, respectively, by (RA1).
  \item[](Case 2) If $\{\omega\}\not\gg_2\{\omega'\}$ and  $\{\omega'\}\not\gg_2\{\omega\}$ then $I_2(\{\omega\})= I_2(\{\omega'\})$ is holding in $G_2$ by Proposition~\ref{pro:cba-to}. Therefore, $\{\omega\}\gg_3\{\omega'\}$ iff $\{\omega\}\gg_1\{\omega'\}$ by (RA4). 
  \end{itemize}
  Therefore, we have $(\{\omega\},\{\omega'\})\in G_3$ iff $(\{\omega\},\{\omega'\})\in G_2$ or $(\{\omega\},\{\omega'\})\in \Lambda(G_1,G_2)$. That is to say, $(\{\omega\},\{\omega'\})\in G_3$ iff $(\{\omega\},\{\omega'\})\in G$. Recall that $G$ and $G_3$ are both CBAs. Then we have $G_3=G$ by Corollary~\ref{cor:cba-to}.
  \end{proof}
This above theorem generalizes Theorem~\ref{thm:totuni} and shows that, under the postulates (RA1)-(RA4), the revision operator $\bullet$ is deterministic when applied to complete belief algebras. Specifically, the revision result $G_1 \bullet G_2$ is uniquely determined by combining the preference relations from $G_1$ and $G_2$ in a principled way. The set $\Lambda(G_1, G_2)$ captures the preferences from $G_1$ that are consistent with the structure of $G_2$, ensuring that the revision process preserves as much of the original belief information as possible while fully incorporating the new evidence. The following corollary establishes a direct correspondence between the revision operator $\bullet$ for complete belief algebras and the revision operator $\circ$ for their equivalent total preorders. 
  \begin{corollary}\label{cor-revpre}
  Suppose, $G_1$ and $ G_2$ are complete belief algebras, and $\preceq_1$ and $\preceq_2$ are their equivalent total preorders on worlds, respectively. If $\bullet$ satisfies (RA1)-(RA4), and $\circ$ satisfies (RE1)-(RE3), then the corresponding total preorder of $G_1\bullet G_2$ is exactly $\preceq_1\circ\preceq_2$.
  \end{corollary}
  \begin{proof}
    This conclusion follows intuitively, and we only provide a proof sketch.

  Let $\preceq=\preceq_1\circ\preceq_2$ and $G=G_1\bullet G_2$. Note that by the proof of Theorem~\ref{thm:rcba}, the revision operator $\bullet$ under (RA1)-(RA4) ensures that in $G$ $\{\omega\}\gg\{\omega^*\}$ iff either $\{\omega\}\gg_2\{\omega^*\}$ (from $G_2$), or $I_2(\{\omega\})=I_2(\{\omega^*\})$ and $\{\omega\}\gg_1\{\omega^*\}$ (preserved from $G_1$).
  
  On the other hand, by the proof of Theorem~\ref{thm:totuni} the revision operator $\circ$ defined by (RE1)-(RE3) ensures that $\omega\prec\omega^*$ iff either $\omega\prec_2\omega^*$ (from $\preceq_2$) or $\omega\sim_2\omega^*$ and $\omega\prec_1\omega^*$ (preserved from $\preceq_1$).

  This structural correspondence guarantees that $G$ is exactly the CBA corresponding to $\preceq_1\circ\preceq_2$ by Lemma~\ref{pro:cba-to}.
  \end{proof}
  
  \begin{example}
  Suppose, $\Delta_1=\{\{\omega_1\}\gg\{\omega_2\}\gg\{\omega_3\}\gg\{\omega_4\}\}$ and $\Delta_2=\{\{\omega_2\}\gg\{\omega_1,\omega_3\}\gg\{\omega_4\}\}$ are
  backbones, then $G_1=\Gen(\Delta_1)$ and $G_2=\Gen(\Delta_2\cup\{(1,4),(3,4)\})$ are complete belief algebras. In this sense, $\Lambda(G_1,G_2)=\{(1,3)\}$, and $G_1\bullet G_2=\Gen(\{\{\omega_2\}\gg\{\omega_1\}\gg\{\omega_3\}\gg\{\omega_4\}\})$, where $\{\{\omega_2\}\gg\{\omega_1\}\gg\{\omega_3\}\gg\{\omega_4\}\}$ is
  the backbone of $G_1\bullet G_2$.
  \end{example}
  
  For the case where $G_1$ and $G_2$ are possibly incomplete belief algebras, we include the following postulate:
  \begin{enumerate}
  \item[](RA5) If $G_1\leq G_1'$, $G_2\leq G_2'$ then $G_1\bullet G_2\subseteq G_1'\bullet G_2'$.
  \end{enumerate}
  Recall that $G_1\leq G_1'$ means that $G_1\subseteq G_1'$ and $G_1$ and $G_1'$ have the same backbone. 
  (RA5) means that if $G_1$ and $G_2$ contain less information than $G'_1$ and $G'_2$, respectively, and they have the same core belief information (i.e., same backbones), 
  then  $G_1\bullet G_2$ also contains less belief information than $G'_1\bullet G'_2$. From (RA5), the following proposition is easy to verify.
  \begin{proposition}
  If $\bullet$ satisfies (RA5), then $G_1\bullet G_2\subseteq \Com(G_1)\bullet \Com(G_2)$.
  \end{proposition}
  The above proposition shows that $\Com(G_1)\bullet \Com(G_2)$ is an upper bound of $G_1\bullet G_2$ (under $\subseteq$). 
  
  Furthermore, we suppose a rational agent will keep maximal information from $G_1$ to $G_1\bullet G_2$ under (RA1)--(RA5), which is a ``minimal change'' rule.  Hence, we assume that $G_1\bullet G_2$ is maximum in $\Com(G_1)\bullet \Com(G_2)$, i.e., 
  there is no belief algebra  
  $\Gen(\Omega) \subseteq \Com(G_1)\bullet \Com(G_2)$
  such that $G_1\bullet G_2\subset \Gen(\Omega)$, where $\Omega \subseteq G_1 \cup G_2$. 
  Then we have the following postulate.
  \begin{enumerate}
  \item[](RA6) If $\Omega\subseteq G_1\cup G_2$ and $\Gen(\Omega)\subseteq \Com(G_1)\bullet \Com(G_2)$, $G_1\bullet G_2\subseteq \Gen(\Omega)$, then $G_1\bullet G_2= \Gen(\Omega)$.
  \end{enumerate}
  
  We arrive to one of the major results in this work:
  \begin{theorem}\label{thm:rev}
  Suppose $\bullet$ satisfies (RA1)-(RA6), and $G_1,G_2$ are belief algebras. Then the revision result of $G_1\bullet G_2$
  is unique, and  $G_1\bullet G_2=\Gen((G_1\cup G_2)\cap \Com(G_1)\bullet \Com(G_2))$.
  \end{theorem}
  \begin{proof}
  $\Com(G_1)\bullet \Com(G_2)$ is a defined complete belief algebra by Theorem~\ref{thm:rcba} and $\bullet$ satisfies (RA1)-(RA6) .
  Since $(G_1\cup G_2)\cap \Com(G_1)\bullet \Com(G_2)$ is a subset of $\Com(G_1)\bullet \Com(G_2)$, $\Gen((G_1\cup G_2)\cap \Com(G_1)\bullet \Com(G_2))$ is a belief algebra by Corollary~\ref{cor-subalgebra}. By (RA2), there is a $\Omega\subseteq G_1\cup G_2$ such that $G_1\bullet G_2=\Gen(\Omega)$. By (RA5), we have $\Gen(\Omega)\subseteq \Com(G_1)\bullet \Com(G_2)$. Then, $\Omega\subseteq \Com(G_1)\bullet \Com(G_2)$. Furthermore, we have $\Omega\subseteq ((G_1\cup G_2)\cap \Com(G_1)\bullet \Com(G_2))$. As a result, we can conclude that $G_1\bullet G_2=\Gen(\Omega)\subseteq \Gen((G_1\cup G_2)\cap \Com(G_1)\bullet \Com(G_2))$, where $\Gen((G_1\cup G_2)\cap \Com(G_1)\bullet \Com(G_2))$ is a belief algebra. By (RA6), we have  $G_1\bullet G_2=\Gen((G_1\cup G_2)\cap \Com(G_1)\bullet \Com(G_2))$.
  \end{proof}
  The above theorem shows that there is only one operator that satisfies (RA1)--(RA6). The postulate (RA5) provides an upper bound for the revision result. On the other hand, (RA6) imposes a conditional maximality requirement on the revision result, which, together with (RA5), leads to the uniqueness of the revision operator $\bullet$. 
In the next section, we will discuss how to algorithmically compute the revision result of this operator, providing a practical method for performing iterated belief revision in the belief algebra framework.
  
  \section{A Practical Algorithm and Discussion}~\label{sec:alg}
  In this section, we provide a practical algorithm for computing the revision result \( G_1 \bullet G_2 \) based on the postulates (RA1)--(RA6). 
\subsection{Algorithm}
 We begin with a direct characterization of the revision result.
  \begin{proposition}
  Suppose $G_1$ and $G_2$ are belief algebras, $\bullet$ satisfies $(RA1)$--$(RA6)$, and $G_*=\Com(G_1)\bullet \Com(G_2)$.
  Then $G_1\bullet G_2=\Gen((G_1\cap G_*)\cup G_2)$, and $G_1\cap G_*=\{(U,V)\mid U\gg_1 V,$ and $I_*(U)\gg_*I_*(V)\}$.
  \end{proposition}
  \begin{proof}
  By Theorem~\ref{thm:rev}, we have $G_1\bullet G_2=\Gen((G_1\cup G_2)\cap G_*)=\Gen((G_1\cap G_*)\cup(G_2\cap G_*))$. From (RA1) and (RA5), we have $G_2\subseteq G_1\bullet G_2\subseteq G_*$. Then $G_2\cap G_*=G_2$. Then $G_1\bullet G_2=\Gen((G_1\cap G_*)\cup G_2)$. Moreover, as
  $G_*$ is a complete belief algebra, $U\gg_* V$ iff $I_*(U)\gg I_*(V)$ by Lemma~\ref{lem:sup}. Hence, $G_1\cap G_*=\{(U,V)\mid U\gg_1 V,$ and $I_*(U)\gg_*I_*(V)\}$.
  \end{proof}
  
  Therefore, we can use Algorithm~\ref{alg:rev} to get $G_1 \bullet G_2$.
  \begin{algorithm}[htpb]
    \caption{Definite revision on $BA_L$.}
    \label{alg:rev}
    \In{%
      Current belief algebra $G_1$ and
      new evidence $G_2$.
    }
    \Out{Resulting belief algebra $G_1 \bullet G_2$.}
      Calculate $\Com(G_1)$ and $\Com(G_2)$ by Definition~\ref{def:com}\;
      $G_* \leftarrow \Com(G_1)\bullet \Com(G_2)$ by Theorem~\ref{thm:rcba}\;
      $G_1\cap G_* \leftarrow \{(U,V)\mid U\gg_1 V,$ and $I_*(U)\gg_*I_*(V)\}$\;
      $G_1\bullet G_2 \leftarrow \Gen((G_1\cap G_*)\cup G_2)$ by Definition~\ref{dfn:gen}\;
      \Return{$G_1\bullet G_2$.}
    \end{algorithm}

    In the first step of the algorithm, we need to compute $\Com(G_1)$ and $\Com(G_2)$.
    Given a belief algebra $G=(2^W, \gg)$, to calculate $\Com(G)$, we need to obtain the backbone $\{U_1 \gg \cdots \gg U_n\}$ of $G$ first. Let $U_1= \bigcap \{ U \subseteq W \mid U \gg W\setminus U\}$, $W_{i+1} = W_{i} \setminus U_i$ and $W_1 = W$. Then $U_i = \bigcap \{U  \subseteq W_{i+1} \mid U \gg W_{i+1} \setminus U\}$. It can be verified that $\{U_1 \gg \cdots \gg U_n\}$ is indeed the backbone of $G$, and more details can be found in~\cite{meng2015belief}. With the backbone of $G$, $\Com(G) = (2^W, \gg^*)$ can be constructed by defining $\gg^*$ as $\forall U,V \in R_W$, $U \gg^* V$ if $I(U) \gg I(V)$. Then $G_*$ can be computed by Theorem~\ref{thm:rcba}, and thus $G_1 \bullet G_2$ can be obtained accordingly. Note that one can get $\Gen(\gg)$ for some $\gg$ by calculating the closure of $\gg$ under (A1), (A3), and (A4).
    
    The following is an example of applying Algorithm~\ref{alg:rev} to the revision scenario in Example~\ref{eg:dp}.
  \begin{example}
  In Example~\ref{eg:dp}, as Bob's current belief corresponds to a total preorder $\{\omega_1\sim\omega_2\prec\omega_3\sim\omega_4\}$, it can be characterized as a CBA $G_1=(2^W, \gg_1)$, where
  \begin{align*}
  \gg_1 & =\{(1,3),(1,4),(2,3),(2,4),(12,3)(12,4),(12,34),\\
  &\quad(13,4),(14,3),(23,4),(24,3),(123,4),(124,3)\}\\
  &\quad\cup Tr(W).
  \end{align*}
  Here for simplicity, we use a sequence of numbers $i_1i_2\ldots i_k$ to represent the set of worlds $\{\omega_{i_1},\omega_{i_2}\ldots, \omega_{i_k}\}$ ($1\le i_1,\ldots, i_k \le 4$). For example, $(1,3)$ represents $(\{\omega_{1}\}, \{\omega_3\})$ and $(12,3)$ represents $(\{\omega_1,\omega_2\},\{\omega_3\})$.
  The new evidence $\mu$ with worlds $[\mu]=\{\omega_1,\omega_4\}$ can be characterized as $G_2 = (2^W,\gg_2)$, where
  \begin{align*}
   &\gg_2 =\Gen(\{([\mu],[\neg\mu])\})\\
  &=\{(14,23),(14,2),(14,3),(142,3),(143,2)\}\cup Tr(W).
  \end{align*}
  Note that the backbone of $G_2$ is $\{\omega_1,\omega_4\}\gg_2\{\omega_2,\omega_3\}$.
   Then
  \begin{align*}
  \Com(G_2) & =\{(1,2),(1,3),(4,2),
  (4,3),(14,23),(14,2),
  \\
  &\quad (14,3), (12,3),(13,2),(24,3),
  (34,2),\\
  &\quad (142,3),(143,2)\}\cup Tr(W).
  \end{align*}
By Theorem~\ref{thm:rcba}, we know $\Lambda(G_1,\Com(G_2))=\{(1,4),(2,3)\}$, and
  \begin{align*}
  &G_* = \Com(G_1)\bullet \Com(G_2)  \\
  =& G_1\bullet \Com(G_2) = \Gen(\Lambda(G_1,\Com(G_2)) \cup \Com(G_2))\\
  =&\{(1,4),(1,2),(1,3),(4,2),(4,3),(2,3),(1,24),\\
  &(1,23),(1,34),(1,234),(4,23),(14,2),(14,3),(14,23),\\
  &(13,4),(13,2),(13,24),(12,3),(12,4),(12,34),(42,3),\\
  &(43,2),(123,4),(124,3),(134,2)\}\cup Tr(W).\\
  \end{align*}
  Then the backbone of $G_*$ is $\{\{\omega_1\}\gg_*\{\omega_4\}\gg_*\{\omega_2\}\gg_*\{\omega_3\}\}$, and $G_1 \cap G_*$ is
  \begin{align*}
    &G_1 \cap G_* = \{(1,3),(1,4),(2,3),(12,3),(12,4),\\
    &\quad(12,34),(13,4),(14,3),(24,3),(123,4),(124,3)\}\\
    &\quad\cup Tr(W).
  \end{align*}
  Then 
  \begin{align*}
  &G_1\bullet G_2 \\
  &=\Gen(\{(1,3),(1,4),(2,3),(12,3),(12,4),(12,34),\\
  &\quad(13,4),(14,3),(24,3),(123,4),(124,3)\}\cup G_2)\\
  &=\{(1,2),(1,3),(1,4),(2,3),(1,23),(1,24),\\
  &\quad(1,34),(1,234),(12,3),(12,4),(12,34),(13,2),\\
  &\quad(13,4),(13,24),(14,2),(14,3),(14,23),\\
  &\quad(24,3),(123,4),(124,3)\}\cup Tr(W)
  \end{align*}
  The backbone of $G_1\bullet G_2$ is $\Delta_3=\{\omega_1\}\gg_3\{\omega_2,\omega_4\}\gg_3\{\omega_3\}$.
  It is not difficult to verify that $G_1\bullet G_2=\Gen(\Delta_3\cup \{(2,3)\})$. Furthermore,
  $G_1\bullet G_2$ is also equal to $\Gen(\{(1,4),(1,2),(1,3),(2,3)\})$. In other words, $G_1\bullet G_2$ is
  generated by $\{\omega_1\prec\omega_4, \omega_1\prec\omega_2, \omega_1\prec\omega_3, \omega_2\prec\omega_3\}$.
  In Example~\ref{eg:dp}, these orderings are exactly the part that must be maintained under the DP framework.
  On the other hand, if the new evidence is $\{\mu, (\neg b \mid \neg f)\}$, where $[\mu] = \{\omega_1, \omega_4\}$ instead, then the
  revision result under the proposed framework will be $\Gen(\{\{\omega_1\} \gg \{\omega_4\} \gg \{\omega_2\} \gg \{\omega_3\}\})$, because $(\neg b \mid \neg f)$ will induce the preference information $\{\omega_4\} \gg \{\omega_2\}$. 
  \end{example}

  Now we return to the traditional belief revision setting, where the current belief is represented as a total preorder $\preceq$ on possible worlds, and the new evidence is a formula $\mu$. This revision setting can be viewed as revising a complete belief algebra by a new belief algebra generated by the formula $\mu$. Suppose that $\bullet$ is the revision operator on $BA_L$ satisfying (RA1)--(RA6). Then the revision process proceeds as follows:
  \begin{itemize}
    \item (Step 1) Represent $\preceq$ and $\mu$ by belief algebras. Let $G_{1}=(2^W,\gg_{1})$ be the corresponding complete belief algebra of $\preceq$ i.e., $\omega \prec\omega'$ iff $\{\omega\}\gg_{1}\{\omega'\}$. Similarly, $\mu$ can be equivalently represented by $G_{\mu}=\Gen(\gg_{\mu}  := \{([\mu],[\neg\mu])\})$.
    \item (Step 2) Calculate  $G_*$. Note that $\Com(G_{\mu})$ 
    is equivalent to a total preorder $\preceq_{\mu}$ s.t. $\omega\prec_\mu\omega'$ iff $\omega\in [\mu]$ and $\omega'\in[\neg\mu]$. Moreover, $G_*=G_{1}\bullet \Com(G_{\mu})$ is also equivalent to a total preorder $\preceq_*$.  
    Following (RE1)-(RE3) (by Corollary~\ref{cor-revpre}, (RA1)-(RA4) equivalently), we can conclude that the strict part of $\preceq_*$ is as follows.
        \begin{itemize} 
        \item If $\omega_1,\omega_2\in [\mu]$ then $\omega_1\prec_*\omega_2$ iff $\omega_1\prec\omega_2$.
        \item If $\omega_1,\omega_2\in [\neg\mu]$ then $\omega_1\prec_*\omega_2$ iff $\omega_1\prec\omega_2$.
        \item If $\omega_1\in [\mu]$ and $\omega_2\in[\neg\mu]$ then $\omega_1\prec_*\omega_2$. 
        \end{itemize}
    \item (Step 3) Calculate $G_1\cap G_*$. From the result of last step, we can see that 
     \begin{itemize}
     \item If $\omega_1,\omega_2\in [\mu]$, then  $(\{\omega_1\},\{\omega_2\})\in G_1\cap G_*$ iff $(\{\omega_1\},\{\omega_2\})\in G_1$.
     \item If $\omega_1,\omega_2\in [\neg\mu]$, then  $(\{\omega_1\},\{\omega_2\})\in G_1\cap G_*$ iff $(\{\omega_1\},\{\omega_2\})\in G_1$.
     \item If $\omega_1\in [\mu]$, $\omega_2\in[\neg\mu]$ and $(\{\omega_1\},\{\omega_2\})\in G_1$, then $(\{\omega_1\},\{\omega_2\})\in G_1\cap G_*$.
     \end{itemize}
     It should be noted that, the preferences on single worlds in $[\mu]$ and $[\neg\mu]$ remain unchanged after revision, and $\omega \prec \omega'$ in $\preceq$ is also maintained if $\omega \in [\mu]$ and $\omega' \in [\neg\mu]$. 
    \item (Step 4) Calculate $G_{1}\bullet G_{\mu}=\Gen((G_1\cap G_*)\cup G_{\mu})$. Then we get the revision result $G_{1}\bullet G_{\mu}$.
    \end{itemize}
      \begin{remark}
        It is evident that the above revision process attempts to preserve as much information from $\preceq$ as possible, particularly those preferences consistent with $[\mu] \gg [\neg\mu]$. This strategy aligns with the principles of the AGM and DP frameworks. However, unlike these frameworks, the specific preference relations to be preserved are determined by the upper bound $G_1 \bullet \Com(G_\mu)$. Only those preferences consistent with $G_1 \bullet \Com(G_\mu)$  are included in the final result, and they must be included. This is the reason why the revision operator produces a unique result.
      \end{remark}

  \subsection{Discussion}

    Our iterative framework builds upon~\cite{meng2015belief}.  The belief algebra framework naturally extends to scenarios where the agent's beliefs are partial or incomplete. Unlike total preorders, which require a complete ranking of all worlds, belief algebras allow for the representation of preferences over subsets of worlds, even when some preferences are unspecified. This flexibility is particularly useful in real-world applications where the agent may have limited or uncertain information. 
    
    A key distinction with \cite{meng2015belief} is that we provide a deeper analysis of the structure of \emph{belief algebra} and introduce two core revision rules, (RA5) and (RA6). (RA5) imposes a macro-level constraint on belief revision, ensuring that the revision result does not exceed the outcome under complete information when belief information is insufficient. (RA6), on the other hand, requires preserving as much of the original belief information as possible under these constraints. Interestingly, these rules induce a unique revision result. Under our framework, agents with the same belief algebra and evidence will produce identical revision outcomes. While traditional views attribute different revision results to varying operators, we argue that rational agents share highly similar revision operators, and differences arise from their distinct belief algebras.

  \section{Conclusion}\label{sec:conc}
    In this paper, we 
    proposed an iterated belief revision framework based on \emph{belief algebra} where the current belief state, new evidence, and revision results are all represented as belief algebras. Through a deep analysis of the structure of belief algebra and inspired by existing principles of belief revision, we devised natural postulates for rational revision behaviors, including (RA5), which imposes an upper-bound constraint on revision results, ensuring that no revision exceeds the outcome under complete information, and (RA6), which requires preserving as much of the original belief information as possible while satisfying the upper-bound constraint. Interestingly, these postulates induce a unique revision operator, providing a deterministic and principled approach to belief revision. This uniqueness offers a clear guideline for selecting specific revision operators in practical applications.
    
    Moreover, we developed a practical revision algorithm under the new framework, demonstrating its feasibility for real-world use. 
    In future work, we aim to explore efficient methods for representing original belief information (e.g., logical statements or preferences) as belief algebra. Additionally, we will focus on improving the efficiency of our algorithm, reducing its complexity (currently, it is exponential to the number of worlds), and testing its application in specific domains, such as knowledge or rule revision in large language models.

  \section*{Acknowledgments}
  We would like to thank the anonymous reviewers for their invaluable help to improve the paper. 
This work was supported by the National Natural Science Foundation of
China under Grant numbers 61806170 and 62276218; the Fundamental Research Funds for the Central Universities under Grant numbers 2682022ZTPY082 and 2682023ZTPY027; the French National Research Agency under Grant number SA21PD01; and University of Montpellier under Grant number PP21PD01-RM06.

\bibliographystyle{named}
\bibliography{iterated-extend}

\end{document}